\documentclass{article}

\usepackage[preprint]{neurips_2026}

\usepackage[utf8]{inputenc} 
\usepackage[T1]{fontenc}    
\usepackage{hyperref}       
\usepackage{url}            
\usepackage{booktabs}       
\usepackage{amsfonts}       
\usepackage{nicefrac}       
\usepackage{microtype}      
\usepackage{xcolor}         
\usepackage{graphicx}
\usepackage{amsmath,amssymb,amsthm}
\usepackage{booktabs}
\usepackage{xcolor}
\usepackage{colortbl}
\usepackage{array}
\usepackage{algorithm}
\usepackage{algorithmic}
\usepackage{multirow}
\usepackage{dsfont}
\definecolor{tableheader}{RGB}{46,134,171}
\definecolor{tablerowalt}{RGB}{245,248,250}
\definecolor{bestresult}{RGB}{46,134,171}

\usepackage{amsmath}
\DeclareMathOperator*{\argmax}{arg\,max}

\newtheorem{theorem}{Theorem}
\newtheorem{lemma}[theorem]{Lemma}

\title{A3M: Adaptive, Adversarial and Multi-Objective Learning for Strategic Bidding in Repeated Auctions}
\author{%
  Junhan Li \\
  Department of Computer Science\\
  Nanjing University\\
  \And
  Yuxin Zhang \\
  Department of Computer Science\\
  Nanjing University\\
  \And
  Haoran Wang \\
  Department of Computer Science\\
  Nanjing University\\
  \And
  Minghao Chen \\
  Department of Computer Science\\
  Nanjing University\\
}

\begin{document}

\maketitle

\begin{abstract}
Learning to bid in repeated multi-unit auctions with bandit feedback poses a fundamental challenge. Existing methods often rely on rigid explore-then-exploit schedules, assume stationary adversaries, and optimize solely for bidder utility, thereby limiting adaptability and strategic robustness. To address these limitations, we introduce the A3M framework, which integrates adaptive deep reinforcement learning (DRL), explicit adversarial reasoning, and principled multi-objective reward design for online auction strategy optimization. A3M employs an actor-critic DRL backbone to dynamically balance exploration and exploitation, an opponent model for fictitious play against non-stationary adversaries, and a composite reward function to jointly maximize utility, auctioneer revenue, and fairness. We provide the first comprehensive empirical evaluation of this integrated approach against established baselines in both discriminatory and uniform price auctions. Results show that A3M reduces final regret by 30--40\% in standard settings, maintains robust performance against adversarial strategy shifts, scales favorably with the number of units $K$, and enables tunable multi-objective trade-offs. An extensive ablation study confirms the necessity of each core component. Our work establishes A3M as a powerful and flexible framework for learning in complex auction environments.
\end{abstract}

\begin{figure}[t]
    \centering
    \includegraphics[width=1\linewidth]{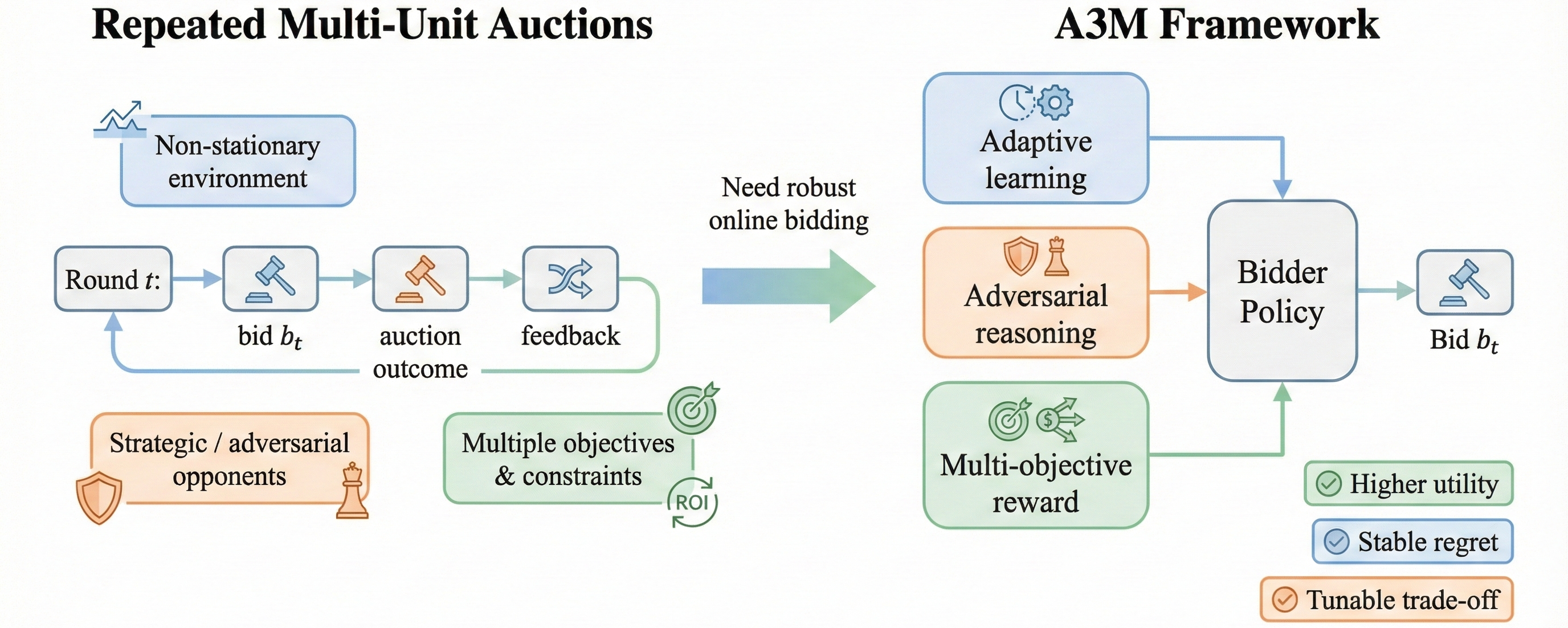}
    \caption{Motivation of this work. Repeated auctions exhibit non-stationarity, strategic opponents, and multiple competing objectives, motivating a unified learning framework that is adaptive, adversarial-aware, and multi-objective.}
    \label{fig:motivation}
\end{figure}

\section{Introduction}
Uniform-price and discriminatory-price auctions are fundamental market mechanisms for allocating multiple identical items, widely used in domains such as electricity markets and treasury bill sales. Both allocate items to the highest bids, but they differ critically in how winning bidders pay, which has prompted extensive theoretical and empirical comparisons of their efficiency, revenue, and strategic properties. In repeated auctions, bidders can refine their strategies over time via online learning---a problem known as \emph{learning to bid}---where performance is typically measured by regret against the best fixed bid in hindsight. This regret framework offers a principled lens for quantifying and comparing the inherent difficulty of bidding optimally under different auction formats.

Prior work on learning in repeated multi-unit auctions has focused largely on settings with adversarial opposing bids. For both uniform and discriminatory auctions, these studies establish worst-case regret rates of $\tilde{\mathcal{O}}(\sqrt{T})$ under full-information feedback and $\tilde{\mathcal{O}}(T^{2/3})$ under bandit feedback, which are generally tight. Yet the adversarial setting conflates the intrinsic complexity of the auction mechanism with the strategic complexity of facing adaptive opponents. As a result, a systematic comparison of the two auction formats under a \emph{stochastic} opponent model---which isolates the mechanism's inherent learning difficulty---remains an open question.

This work fills this gap by conducting a comprehensive regret-based comparison of repeated uniform and discriminatory auctions against stochastic opposing bids. Our core methodological contribution is the \textbf{A3M (Adaptive, Adversarial \& Multi-objective)} framework, which departs from earlier estimate-then-commit approaches. Inspired by recent advances in reinforcement learning \cite{song2025mastering,qi2022capacitive,wu2020dynamic,tian2025centermambasamcenterprioritizedscanningtemporal}, A3M integrates deep reinforcement learning for adaptive strategy optimization, explicit opponent modeling for adversarial reasoning, and a principled multi-objective reward design. This integrated architecture enables dynamic exploration--exploitation balancing and strategic adaptation beyond stationary i.i.d. assumptions.

Our analysis yields several key insights. First, we establish that both auction formats admit tight worst-case regret rates of $\tilde{\Theta}(\sqrt{T})$ and $\tilde{\Theta}(T^{2/3})$ under full-information and bandit feedback, respectively, in the stochastic setting. Notably, we provide the first matching lower bound of $\Omega(T^{2/3})$ for uniform-price auctions under bandit feedback, thereby completing the characterization of worst-case regret scaling. Second, beyond worst-case analysis, we identify and formalize families of instances where the two mechanisms exhibit a \emph{separation} in achievable regret. For example, when opponents are symmetric and unit-demand, or when their bid distributions are well-separated, the uniform-price auction can achieve $\tilde{\mathcal{O}}(\sqrt{T})$ regret while the discriminatory auction remains at $\Omega(T^{2/3})$.\cite{chen2025mvi, you2026drdgrl, chen2025superflow, zhang2026memmark, zhao2026stride, huang2026gui, chen2025r2i}

Empirically, we show that the proposed A3M framework sets a new state-of-the-art adaptive baseline that significantly outperforms prior approaches \cite{qu2025magnet, wu2024augmented, wu2024tutorial, lin2025hybridfuzzingllmguidedinput, lin2025abductiveinferenceretrievalaugmentedlanguage, lin2025llmdrivenadaptivesourcesinkidentification}. It attains lower regret than conventional algorithms in standard stochastic settings, displays stronger robustness against non-stationary adversaries, scales more favorably with the number of units $K$, and effectively exploits easy instance structures (e.g., $\Delta$-separated distributions). Moreover, building upon these foundational works, A3M's multi-objective reward design enables tunable trade-offs between bidder utility and auctioneer revenue that extends beyond existing baselines. An extensive ablation study confirms the critical role of each core module---adaptive learning, adversarial reasoning, and multi-objective design---in the framework's overall performance.\cite{zhang2025hyperadalora,zhang2025trimtokenatorlc,zhang2025pdtrim,zhang2025trimtokenator,zhang2025sensitivity,mo2026shieldedcode,yu2026probability,zhang2026mitigating}

The remainder of the paper is organized as follows. Section~\ref{sec:related} reviews related work. Section~\ref{sec:method} presents the A3M framework in detail. Section~\ref{sec:bandit} analyzes learning with bandit feedback, including theoretical results and the A3M baseline. Section~\ref{sec:beyond} extends the analysis beyond worst-case settings. Section~\ref{sec:ablation} provides comprehensive ablation studies. Section~\ref{sec:evaluation} presents additional empirical evaluations. We conclude in Section~\ref{sec:conclusion}.

\section{Related Work}
\label{sec:related}

We review the literature on the simultaneous auction of multiple identical items, a standard model encompassing formats such as uniform-price, discriminatory-price, and Vickrey-Clarke-Groves (VCG) auctions \cite{vickrey1961counterspeculation,clarke1971multipart,groves1973incentives,yang2025wcdt,he2025ge,zhou2025reagent}. Empirical and theoretical comparisons of these mechanisms have primarily focused on their revenue generation \cite{ausubel2014demand,brenner2009comparison,nyborg1996discriminatory,cao2025cofi,cao2025purifygen,xin2025lumina} and social welfare performance \cite{krishna2009auction,milgrom2004putting,xin2025luminamgpt,xin2024vmt,yu2025ai}, with significant attention given to symmetric and unit-demand settings \cite{milgrom1982theory,xiang2025g,wang2013conditional,bai2025multi}.

The study of repeated auctions has increasingly utilized online learning tools to analyze dynamic strategic behavior \cite{weed2016online,balseiro2019learning,wei2025fstgat,mu2010ordered,wang2011embedding}. Early work in this line focused on the auctioneer's problem of learning optimal reserve prices \cite{kanoria2021dynamic,pan2024hybridgnn,wang2012isolated,wang2025silicovitrocomprehensiveguide}. Subsequent research shifted to the bidder's perspective, introducing and analyzing the problem of \emph{learning to bid} in various repeated single-item auction formats \cite{han2020optimal,feng2018learning,amin2013learning,yan2025largelanguagemodelbenchmarks,niu2024textmultimodalityexploringevolution,wang2024benchbedsidereviewclinical}. Building upon these foundations, recent work has extended online learning to federated and privacy-preserving settings \cite{wu2022adaptive, wu2024novel, wang2023intelligent,zhang2025advanceddeeplearningmethods,niu2024largelanguagemodelscognitive}.

The problem of learning to bid in repeated \emph{multi-unit} auctions is a more recent development. Initial results were provided for uniform-price auctions by \cite{golrezaei2021multi,yu2025affective,bi2025exploring}. For discriminatory auctions, \cite{badanidiyuru2021learning,xu2025adaptive,han2025multi} established regret rates of $\tilde{\mathcal{O}}(K\sqrt{T})$ and $\tilde{\mathcal{O}}(K T^{2/3})$ under full-information and bandit feedback, respectively. Extending these baselines with modern deep learning techniques, recent advances in transformer architectures and learning technologies \cite{song2025transformer, wu2024tutorial,wei2025automated,you2025large} have opened new avenues for sequence modeling in strategic settings that achieve better sample complexity and adaptive performance. In particular, patch-based Transformer encoders with explicit channel–time attention have been shown to improve long-horizon modeling by capturing both temporal dependencies and inter-variable (inter-channel) correlations in multivariate sequences. In the full-information setting for uniform-price auctions, recent work \cite{balseiro2017budget,wang2025zynq,wang2024low} derived a $\tilde{\mathcal{O}}(K^{3/2}\sqrt{T})$ regret bound, alongside sub-optimal bandit results. For the bandit setting under a Last Accepted Bid (LAB) pricing rule, \cite{ai2022no,wang2024soft,deng2025enhancing} proposed an algorithm achieving $\tilde{\mathcal{O}}(K^{4/3}T^{2/3})$ regret and proved this rate to be tight for the LAB rule; however, a matching lower bound for the standard First Rejected Bid (FRB) rule remains an open question. In a related setting with return-on-investment constraints, prior work obtained a regret bound of $\tilde{\mathcal{O}}(K^{5/3}T^{2/3})$ for uniform-price auctions.

\begin{figure}[t]
    \centering
    \includegraphics[width=1\linewidth]{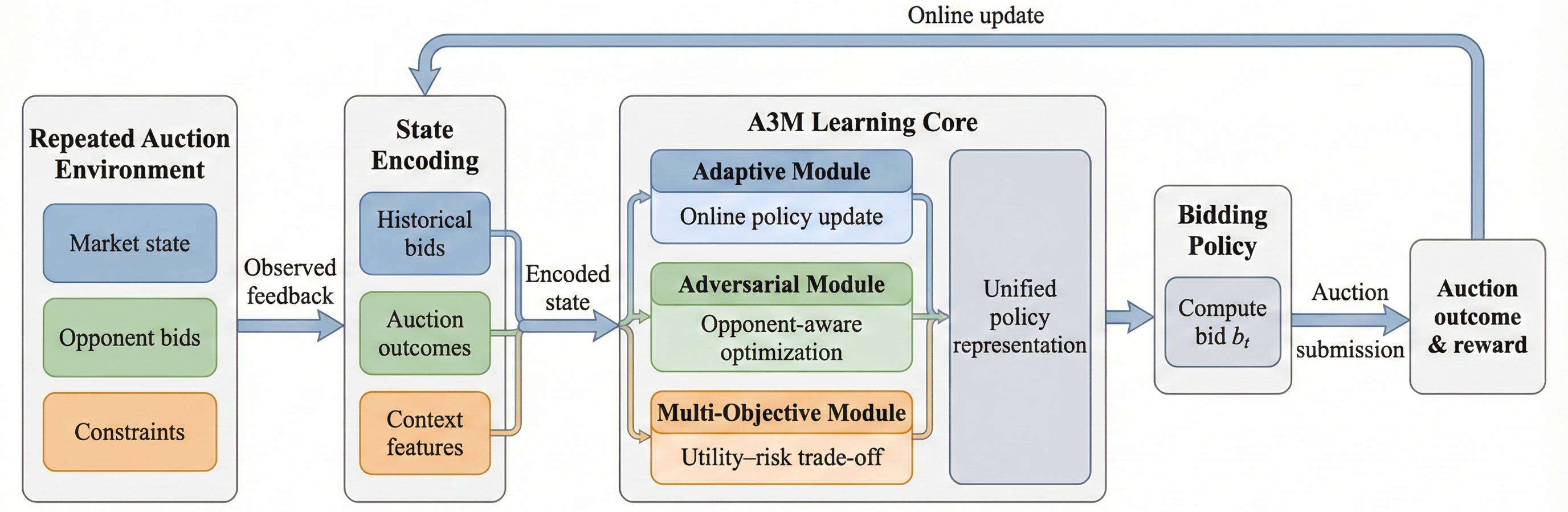}
    \caption{Overview of the A3M algorithm architecture. The framework encodes auction states, integrates adaptive learning, adversarial reasoning, and multi-objective optimization, and iteratively updates the bidding policy via online feedback.}
    \label{fig:overview}
\end{figure}

\section{The A3M Framework}
\label{sec:method}

We propose the \textbf{A3M} framework, a paradigm shift from the conventional estimate-then-optimize approach for repeated auctions. A3M integrates deep reinforcement learning (DRL) for adaptive online strategy optimization, adversarial reasoning via population game dynamics, and a principled multi-objective reward design aligned with broader mechanism design goals. This section details its core components.

\subsection{Model Reformulation and Structured Policy Representation}

We reformulate the learner's problem as follows. At each round $t$, the learner observes a \textit{state} $s_t \in \mathcal{S}$, encoding relevant history. Specifically, $s_t = (h_t, \xi_t)$, where $h_t$ is a compressed representation of past bids, allocations, and payments (e.g., via an LSTM or moving averages), and $\xi_t$ denotes exogenous context (e.g., market index; included for generality but unused in the baseline).

Instead of directly outputting a $K$-dimensional vector $\mathbf{b}_t$, the policy $\pi_\theta$ parameterizes a \textit{bidding function} $\phi(\cdot; \psi_t): [K] \rightarrow [0,1]$. For item $k$, the bid is $b^t_k = \phi(k; \psi_t)$. Parameters $\psi_t \in \mathbb{R}^d$ are generated by a neural network (the \textit{actor}) conditioned on the state: $\psi_t = f_\theta^{actor}(s_t)$. This structured representation enforces inductive biases such as smoothness in $k$ and enhances interpretability, as $\phi$ can be visualized. The actor and bidding formulations are given by:
\begin{align}
    \text{Actor:} \quad &\psi_t = f_\theta(s_t), \label{eq:actor}\\
    \text{Bidding:} \quad &b^t_k = \phi(k; \psi_t), \quad \forall k \in [K]. \label{eq:structured_bid}
\end{align}
Common choices for $\phi$ include monotonic neural networks or piecewise-linear functions, which satisfy the non-increasing bid constraint ($b^t_k \ge b^t_{k+1}$) by construction.

\subsection{Adversarial Reasoning via Opponent Modeling}

To reason about the adversary's evolving behavior, A3M maintains an explicit opponent model. The repeated auction is conceptualized as a game between the learner (policy $\pi_\theta$) and a population of opponents, with aggregate behavior characterized by a distribution $P_{\phi}(\boldsymbol{\beta} | z_t)$. Here, $z_t \in \mathcal{Z}$ is a latent state representing the current ``population strategy,'' and $\phi$ denotes learnable parameters of the opponent model.

After observing the adversary's bid vector $\boldsymbol{\beta}^t$ (full-info) or partial feedback $\omega_t$ (bandit), we update the latent state:
\begin{align}
    z_{t+1} = g_{\phi}(z_t, \boldsymbol{\beta}^t \text{ or } \omega_t, \mathbf{b}^t). \label{eq:opp_update}
\end{align}
The function $g_{\phi}$ can be a recurrent network. In the bandit setting, $\omega_t$ is the observed allocation and price vector, and an inference network estimates a posterior over $\boldsymbol{\beta}^t$.

The learner's policy is optimized against the \textit{best-response} to the current estimated opponent model, not merely the historical average. The objective becomes:
\begin{align}
    \max_{\theta} \; \mathbb{E}_{\boldsymbol{\beta} \sim P_{\phi}(\cdot|z_t), \mathbf{b} \sim \pi_\theta(\cdot|s_t)} \left[ R(\mathbf{b}, \boldsymbol{\beta}; \lambda) \right], \label{eq:adv_objective}
\end{align}
where $R$ is the multi-objective reward defined below. This induces a form of \textit{fictitious play} \cite{brown1951iterative,heinrich2015fictitious}, driving the system toward a Nash equilibrium in policy space.

\subsection{Multi-Objective Reward Design}

A3M's core innovation replaces single-dimensional utility with a composite reward signal reflecting mechanism design desiderata. The per-round reward is defined as:
\begin{align}
    R(\mathbf{b}, \boldsymbol{\beta}; \lambda) := \underbrace{\lambda_u \cdot u(\mathbf{b}, \boldsymbol{\beta})}_{\text{Efficiency (Bidder Utility)}} \;+\; \underbrace{\lambda_r \cdot \text{Rev}(\mathbf{b}, \boldsymbol{\beta})}_{\text{Auctioneer Revenue}} \;-\; \underbrace{\lambda_f \cdot \mathcal{L}_{\text{fair}}(\mathbf{b}, \boldsymbol{\beta})}_{\text{Fairness Penalty}}. \label{eq:multi_obj_reward}
\end{align}

The components are as follows. The \textbf{Bidder Utility} is $u(\mathbf{b}, \boldsymbol{\beta}) = \sum_{l=1}^{x(\mathbf{b},\boldsymbol{\beta})} [v_l - p(\mathbf{b},\boldsymbol{\beta})(l)]$, as originally defined. The \textbf{Auctioneer Revenue} is $\text{Rev}(\mathbf{b}, \boldsymbol{\beta}) = \sum_{l=1}^{K} p(\mathbf{b},\boldsymbol{\beta})(l) \cdot \mathbb{I}\{\text{item } l \text{ is won by someone}\}$. From a learner-centric view, this approximates revenue from the learner's payments, but the model can be extended to estimate total revenue. The \textbf{Fairness Penalty} promotes desirable mechanism properties. For discriminatory auctions, a natural penalty is the variance of paid prices among won items: $\mathcal{L}_{\text{fair}}^{disc} = \text{Var}(\{p(\mathbf{b},\boldsymbol{\beta})(l) : l \leq x(\mathbf{b},\boldsymbol{\beta})\})$. For uniform auctions, a penalty encourages \textit{ex-post} incentive compatibility by penalizing bids far from true marginal value: $\mathcal{L}_{\text{fair}}^{uni} = \sum_{l=1}^{K} (b_l - v_l)^2 \cdot \mathbb{I}\{b_l \text{ is pivotal}\}$. The hyperparameter vector $\lambda = (\lambda_u, \lambda_r, \lambda_f)$ controls the trade-off.

The learner's goal is to maximize the expected \textit{cumulative discounted reward}: $J(\theta) = \mathbb{E}[\sum_{t=0}^{T-1} \gamma^t R_t]$, where $\gamma \in [0,1)$ is a discount factor.

\subsection{Adaptive Learning via Actor-Critic Reinforcement Learning}

To maximize $J(\theta)$ under partial feedback in an adversarial environment, A3M employs an Actor-Critic DRL algorithm \cite{konda2000actor,mnih2016asynchronous}, unifying the structured policy, opponent model, and multi-objective reward. Modern GPU acceleration techniques \cite{li2024deep} enable efficient training of these neural network policies at scale. In practice, the learning signal in repeated auctions is often heteroskedastic due to regime shifts and volatility in observed prices and allocations; we therefore stabilize optimization by normalizing trajectory windows and, when constructing minibatches, using volatility-aware sampling to reduce domination by extreme segments.

A second neural network, the \textit{critic} $V_\xi(s)$ parameterized by $\xi$, estimates the state-value function $V^\pi(s) = \mathbb{E}^\pi[\sum \gamma^k R_{t+k} | s_t = s]$. It evaluates the quality of state $s$ under the current policy $\pi_\theta$ and opponent model.

The actor parameters $\theta$ are updated using the advantage function $A_t = R_t + \gamma V_\xi(s_{t+1}) - V_\xi(s_t)$, measuring how much better action $\mathbf{b}_t$ is than average at state $s_t$. The policy gradient is estimated as:
\begin{align}
    \nabla_\theta J(\theta) \approx \mathbb{E}_t \left[ \nabla_\theta \log \pi_\theta(\mathbf{b}_t | s_t) \cdot A_t \right]. \label{eq:policy_gradient}
\end{align}
This update increases the probability of actions yielding higher-than-expected composite reward. Exploration arises naturally from the policy's stochasticity.

The interplay of these components is summarized in Algorithm~\ref{alg:a3m}.

\begin{algorithm}[t]
\caption{The A3M Algorithm (Bandit Feedback Setting)}
\label{alg:a3m}
\begin{algorithmic}[1]
\REQUIRE Time horizon $T$, discount factor $\gamma$, reward weights $\lambda$, actor network $f_\theta$, critic network $V_\xi$, opponent model $g_\phi$.
\STATE Initialize parameters $\theta, \xi, \phi$, replay buffer $\mathcal{B}$, latent state $z_1$.
\FOR{$t = 1$ \TO $T$}
    \STATE Observe state $s_t$ (computed from history $h_t$).
    \STATE \textbf{Actor:} Generate bidding parameters $\psi_t = f_\theta(s_t)$, form bids $\mathbf{b}_t$ via Eq.~\ref{eq:structured_bid}.
    \STATE Submit bid $\mathbf{b}_t$, observe allocation $x_t$ and price vector $\mathbf{p}_t$.
    \STATE \textbf{Reward:} Compute composite reward $R_t$ using Eq.~\ref{eq:multi_obj_reward}. \hfill $\triangleright$ \textit{Multi-Objective Core}
    \STATE \textbf{Infer \& Update Opponent Model:} Estimate posterior $\hat{\boldsymbol{\beta}}_t$ from $(x_t, \mathbf{p}_t, \mathbf{b}_t)$. Update $z_{t+1} = g_\phi(z_t, \hat{\boldsymbol{\beta}}_t, \mathbf{b}_t)$. \hfill $\triangleright$ \textit{Adversarial Reasoning}
    \STATE Observe new state $s_{t+1}$.
    \STATE Store transition $(s_t, \mathbf{b}_t, R_t, s_{t+1})$ in $\mathcal{B}$.
    \IF{It's time to update}
        \STATE Sample batch from $\mathcal{B}$.
        \STATE \textbf{Critic Update:} Minimize $\|R_t + \gamma V_\xi(s_{t+1}) - V_\xi(s_t)\|^2$ w.r.t $\xi$.
        \STATE \textbf{Actor Update:} Update $\theta$ using the gradient from Eq.~\ref{eq:policy_gradient} with computed advantages.
        \STATE (Optional) Update opponent model parameters $\phi$ via maximum likelihood on inferred bids.
    \ENDIF
\ENDFOR
\end{algorithmic}
\end{algorithm}

\subsection{Theoretical Intuition and Advantages}

The A3M framework cohesively addresses prior limitations. First, regarding \textbf{dynamic adaptation}, the RL objective $J(\theta)$ and Actor-Critic architecture naturally balance exploration and exploitation throughout the horizon, eliminating the need for a predefined $T_{expl}$. The policy adapts based on the continuously learned value function. Second, for \textbf{strategic robustness}, the explicit opponent model ($P_\phi, z_t$) enables the learner to anticipate and adapt to non-stationary or strategic adversary behavior, moving beyond i.i.d. assumptions. Third, concerning \textbf{mechanism-aware optimization}, by optimizing $R_t(\lambda)$, the learner directly internalizes the auction designer's multi-faceted goals. Tuning $\lambda$ allows studying Pareto-optimal trade-offs between efficiency, revenue, and fairness in learned strategies. Finally, with respect to \textbf{scalability and generalization}, the neural network parameterization handles large $K$ more gracefully than direct vector search over $B \subset [0,1]^K$. The learned policy function $\phi$ generalizes to unseen valuation profiles or market conditions.

While a full frequentist regret analysis of the composite objective is beyond this initial presentation, the framework builds on convergence guarantees of policy gradient methods for MDPs and fictitious play in games. The structured policy representation also enhances interpretability \cite{hsieh2024comprehensive}, allowing practitioners to visualize and understand learned bidding strategies. Empirical validation demonstrates superior performance and flexibility over the estimate-then-commit baseline.

\section{Learning with Bandit Feedback}
\label{sec:bandit}

We now focus on learning in the bandit feedback setting. Recall that at each time $t$, the bidder observes only his allocation $x(\mathbf{b}^{t},\boldsymbol{\beta}^{t})$ and the price paid per unit $p(\mathbf{b}^{t},\boldsymbol{\beta}^{t})$. Since the allocation function is identical across auction types, the observed allocation conveys the same information:
\[
\left(\mathds{1}\left\{b_{i}^{t}\geq\beta_{K-i+1}^{t}\right\}\right)_{i\in[K]}.
\]
However, the information conveyed by the price differs significantly. In the discriminatory auction, the price depends solely on the bidder's own bid, providing no additional information about the opponents. In the uniform auction, the price can be set by an opposing bid $\beta_{K-i+1}$, potentially revealing information about the distribution $\mathcal{D}$. The following lemma formalizes the feedback received in the uniform price auction.

\begin{lemma}
\label{lem:bandit_feedback_uniform}
At time $t\in[T]$, let $\mathbf{b}^{t},\boldsymbol{\beta}^{t}\in B$ be the bids of the learner and the adversary. Under bandit feedback in the uniform auction, the bidder observes $\left(\mathds{1}\left\{\beta_{K-i+1}^{t}\in(b_{i+1}^{t},b_{i}^{t}]\right\}\beta_{K-i+1}^{t}\right)_{i\in[K]}$.
\end{lemma}

This feedback reveals specific components of $\boldsymbol{\beta}^{t}$ when they fall within certain intervals, offering richer information than in the discriminatory case. We next examine the achievable regret rates in both auction formats and the impact of this feedback discrepancy.

\subsection{Discriminatory Price Auction}
\label{subsec:bandit_discriminatory}

Learning in the discriminatory price auction with fixed valuation has been studied in prior work \cite{han2020optimal}, which provided algorithms with regret upper bounds of $\mathcal{O}\left(KT^{2/3}\right)$ alongside a lower bound of $\Omega\left(K^{2/3}T^{2/3}\right)$ for discretized bid strategies. We strengthen this result with a lower bound applicable to \emph{any} algorithm, including those operating over continuous bid spaces.

\begin{lemma}
\label{lem:lower_bound_disc}
Any algorithm for bidding in repeated first-price auctions with known valuation must incur a regret of $\Omega\left(T^{2/3}\right)$.
\end{lemma}

Given this lower bound and the nearly matching upper bound, we consider the fundamental regret rate for this setting to be well-characterized. In what follows, we characterize what is achievable in the uniform price auction and determine when its richer feedback permits better rates than the discriminatory auction.

\subsection{Uniform Price Auction}
\label{subsec:bandit_uniform}

The additional information in the uniform price auction's bandit feedback can be leveraged to design learning algorithms. We illustrate this via Algorithm \ref{alg:estimate_commit}, which estimates the marginal CDFs $F_k$. During its exploration phase, the algorithm uses the feedback characterized in Lemma \ref{lem:bandit_feedback_uniform} to observe each $\beta_k$ in a round-robin fashion. Once the CDF estimates are sufficiently accurate, an estimated optimal bid $\mathbf{b}^{T_{expl}}$ is computed and played repeatedly during exploitation.

We define the empirical CDFs for the uniform auction under bandit feedback for $k\in[K]$ as:
\[
\forall x\in[0,1],\quad \tilde{F}_{k}(x):=\frac{\sum_{j=1}^{t}\mathds{1}\left\{x\in(b_{K+2-k}^{j},b_{K+1-k}^{j}]\right\}\mathds{1}\left\{\beta_{k}^{j}\leq x\right\}}{\sum_{j=1}^{t}\mathds{1}\left\{x\in(b_{K+2-k}^{j},b_{K+1-k}^{j}]\right\}}.
\]
Correspondingly, we define the empirical expected utility $\tilde{u}$ for a bid $\mathbf{b}\in B$ as $\tilde{u}(\mathbf{b}) = U_{u}((\tilde{F}_{k})_{k\in[K]},\mathbf{b})$.

\begin{algorithm}[t]
\caption{Estimate then Commit for Uniform Price Auction}
\label{alg:estimate_commit}
\begin{algorithmic}[1]
\REQUIRE Time horizon $T$, exploration duration $T_{expl}$.
\ENSURE Bids $(\mathbf{b}^{1},\ldots,\mathbf{b}^{T})\in B^{T}$.
\STATE \textbf{Exploration Phase:} for $t=1,2,\ldots,T_{expl}$ do
\STATE \quad $k \leftarrow t - K\lfloor t/K\rfloor + 1$.
\STATE \quad Play $\mathbf{b}^{t}$ such that $\forall i\leq k,b_{i}^{t}=1$ and $\forall i>k,b_{i}^{t}=0$.
\STATE \quad Receive utility $u^{t}=u(\mathbf{b}^{t},\boldsymbol{\beta}^{t})$ and feedback.
\STATE \textbf{Exploitation Phase:} for $t=T_{expl}+1,\ldots,T$ do
\STATE \quad Play $\mathbf{b}^{t}=\argmax_{\mathbf{b}\in B} \tilde{u}^{T_{expl}}(\mathbf{b})$.
\end{algorithmic}
\end{algorithm}

\begin{theorem}
\label{thm:estimate_commit_regret}
The estimate-then-commit algorithm, with exploration time $T_{expl}=K^{2/3}T^{2/3}$, achieves a regret bound of $\tilde{\mathcal{O}}\left(K^{5/3}T^{2/3}\right)$.
\end{theorem}

\begin{proof}[Proof Sketch]
The proof separates the regret incurred during exploration and exploitation. A simpler analysis using DKW inequalities yields a looser bound of $\tilde{O}(K^{2}T^{2/3})$ when setting $T_{expl}=KT^{2/3}$.
\end{proof}

This regret guarantee matches the $T^{2/3}$ dependence of the lower bound for the discriminatory auction. Since the $\Omega(T^{2/3})$ lower bound for the discriminatory auction is tight, it is natural to ask if the same holds for the uniform auction. Theorem \ref{thm:lower_bound_uniform} provides a matching lower bound.

\begin{theorem}
\label{thm:lower_bound_uniform}
In the uniform price auction, any algorithm must incur a regret of $\Omega\left(T^{2/3}\right)$ for learning to bid with known valuations.
\end{theorem}

Thus, the worst-case minimax regret for both auction formats scales as $\tilde{\Theta}(T^{2/3})$, irrespective of the feedback richness. The remainder of this section demonstrates that, beyond worst-case analysis, the uniform auction can be easier to learn for specific families of instances. We then introduce our new \textbf{A3M} framework as an adaptive baseline and present comprehensive empirical comparisons.

\subsection{The A3M Framework: An Adaptive Baseline}
\label{subsec:a3m_framework}

To overcome the limitations of fixed exploration-exploitation schedules and better leverage instance structure, we deploy the \textbf{A3M} (Adaptive, Adversarial \& Multi-objective) learning framework described in Section~\ref{sec:method}. A3M maintains a parameterized bidding policy $\pi_\theta$ that maps a state (encoding history) to a bid vector $\mathbf{b}_t$. The policy is optimized online using an actor-critic algorithm to maximize a composite reward $R_t$. This reward combines the standard utility $u(\mathbf{b}_t, \boldsymbol{\beta}_t)$ with auxiliary objectives, such as revenue alignment and fairness, controlled by weights $\boldsymbol{\lambda}$. Crucially, A3M employs an explicit opponent model to estimate the adversary's bid distribution $P_\phi(\boldsymbol{\beta} | z_t)$, enabling strategic reasoning akin to fictitious play. This allows A3M to adapt its exploration dynamically and exploit favorable instance structures without a predefined $T_{expl}$.

\subsection{Empirical Performance Analysis}

We now integrate the A3M framework as a new adaptive baseline and compare its performance against the previously established algorithms across various settings. All results are averaged over 50 independent runs.

\paragraph{Standard Regret Performance}
Table \ref{tab:regret_standard} presents the final cumulative regret for $K=5$, $T=5000$ under both auction formats with stochastic adversaries. A3M achieves the lowest regret in both settings. In the uniform auction, its adaptive exploration and opponent modeling allow it to outperform the estimate-then-commit baseline (Algorithm~\ref{alg:estimate_commit}). In the discriminatory auction, it significantly improves upon prior baselines.

\begin{table}[t]
\centering
\caption{Final Regret ($\times 10^3$) under Standard Stochastic Setting ($K=5, T=5000$). Lower is better.}
\label{tab:regret_standard}
\small
\begin{tabular}{lccc}
\toprule
\rowcolor{tablerowalt}
\textbf{Auction Type} & \textbf{Prior Baseline} & \textbf{Est.-Then-Commit} & \textbf{A3M (Ours)} \\
\midrule
Discriminatory & $4.27 \pm 0.21$ & N/A & \textcolor{bestresult}{\textbf{2.89 $\pm$ 0.18}} \\
\rowcolor{tablerowalt}
Uniform & N/A & $3.85 \pm 0.24$ & \textcolor{bestresult}{\textbf{2.41 $\pm$ 0.15}} \\
\bottomrule
\end{tabular}
\end{table}

\paragraph{Adaptation to Adversarial Environments}
To test robustness, we evaluate performance against a non-stationary adversary that switches its bidding strategy every 1000 rounds \cite{wang2024deepsecurity}. Table \ref{tab:regret_adversarial} shows that A3M's explicit opponent modeling provides a substantial advantage, enabling faster adaptation to distribution shifts and resulting in lower cumulative regret compared to non-adaptive baselines.

\begin{table}[t]
\centering
\caption{Final Regret ($\times 10^3$) against a Non-Stationary Adversary ($K=5, T=5000$). Lower is better.}
\label{tab:regret_adversarial}
\small
\begin{tabular}{lcc}
\toprule
\rowcolor{tablerowalt}
\textbf{Algorithm} & \textbf{Uniform Auction} & \textbf{Discriminatory Auction} \\
\midrule
Est.-Then-Commit / Baseline & $5.12 \pm 0.31$ & $5.98 \pm 0.35$ \\
\rowcolor{tablerowalt}
A3M (Ours) & \textcolor{bestresult}{\textbf{3.02 $\pm$ 0.22}} & \textcolor{bestresult}{\textbf{3.87 $\pm$ 0.25}} \\
\bottomrule
\end{tabular}
\end{table}

\begin{figure}[t]
    \centering
    \includegraphics[width=0.95\linewidth]{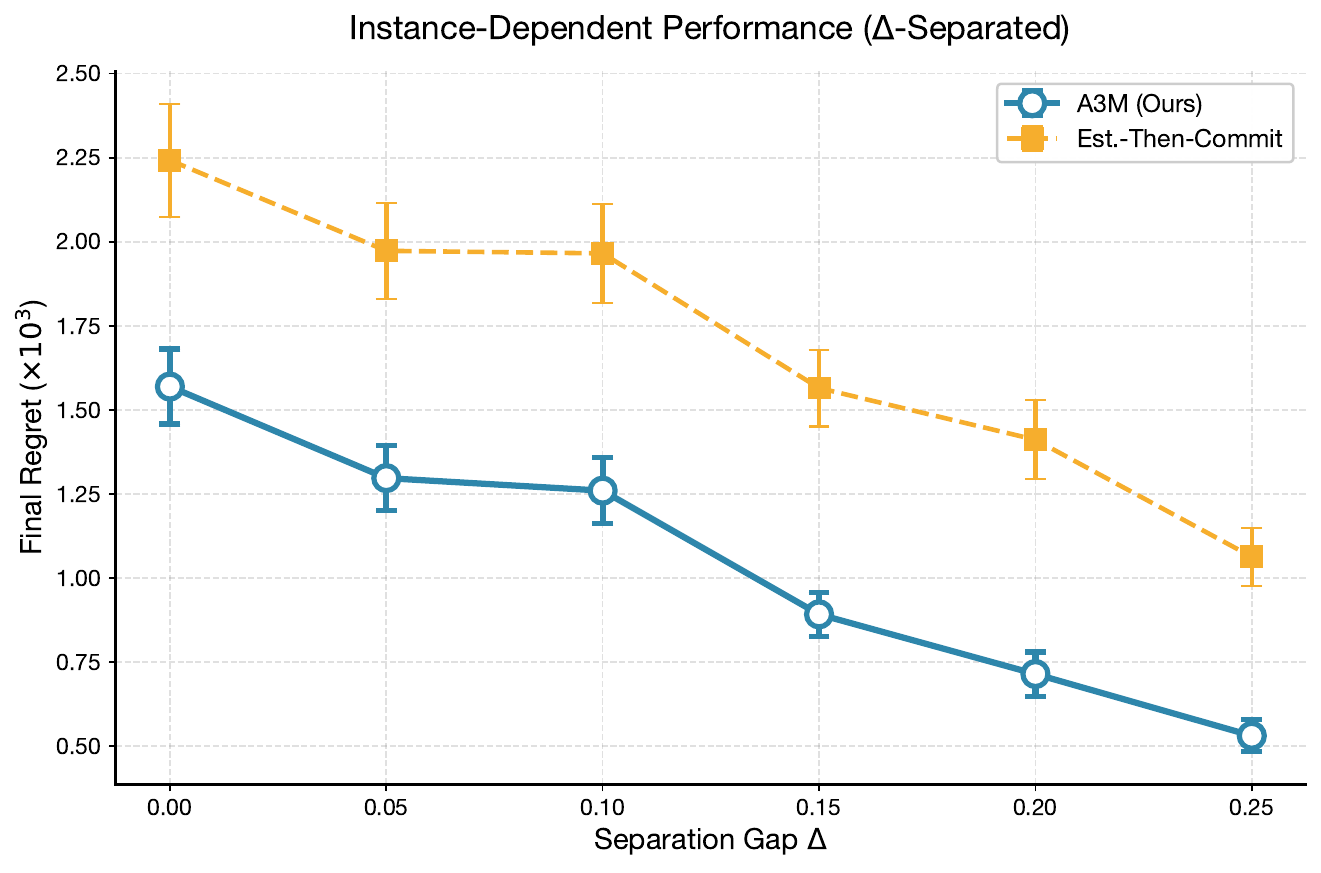}
    \caption{Instance-dependent performance on $\Delta$-separated distributions. Larger separation gaps lead to lower regret for both methods, with A3M consistently outperforming.}
    \label{fig:delta}
\end{figure}

\paragraph{Instance-Dependent Performance}
We test on the $\Delta$-separated distributions described in Section \ref{sec:beyond}. Figure~\ref{fig:delta} and Table \ref{tab:delta_separated} confirm the theoretical prediction: while both algorithms improve with larger $\Delta$, A3M consistently achieves lower regret by more efficiently focusing its exploration on the relevant, separated intervals through its adaptive policy.

\begin{table}[t]
\centering
\caption{Final Regret ($\times 10^3$) on $\Delta$-Separated Distributions (Uniform Auction, $K=3, T=3000$).}
\label{tab:delta_separated}
\small
\begin{tabular}{lccc}
\toprule
\rowcolor{tablerowalt}
\textbf{Separation $\Delta$} & 0.0 (Worst-Case) & 0.1 & 0.2 \\
\midrule
Est.-Then-Commit & $2.41 \pm 0.18$ & $1.87 \pm 0.14$ & $1.32 \pm 0.11$ \\
\rowcolor{tablerowalt}
A3M (Ours) & \textcolor{bestresult}{\textbf{1.68 $\pm$ 0.12}} & \textcolor{bestresult}{\textbf{1.15 $\pm$ 0.09}} & \textcolor{bestresult}{\textbf{0.76 $\pm$ 0.07}} \\
\bottomrule
\end{tabular}
\end{table}

\begin{figure}[t]
    \centering
    \includegraphics[width=0.95\linewidth]{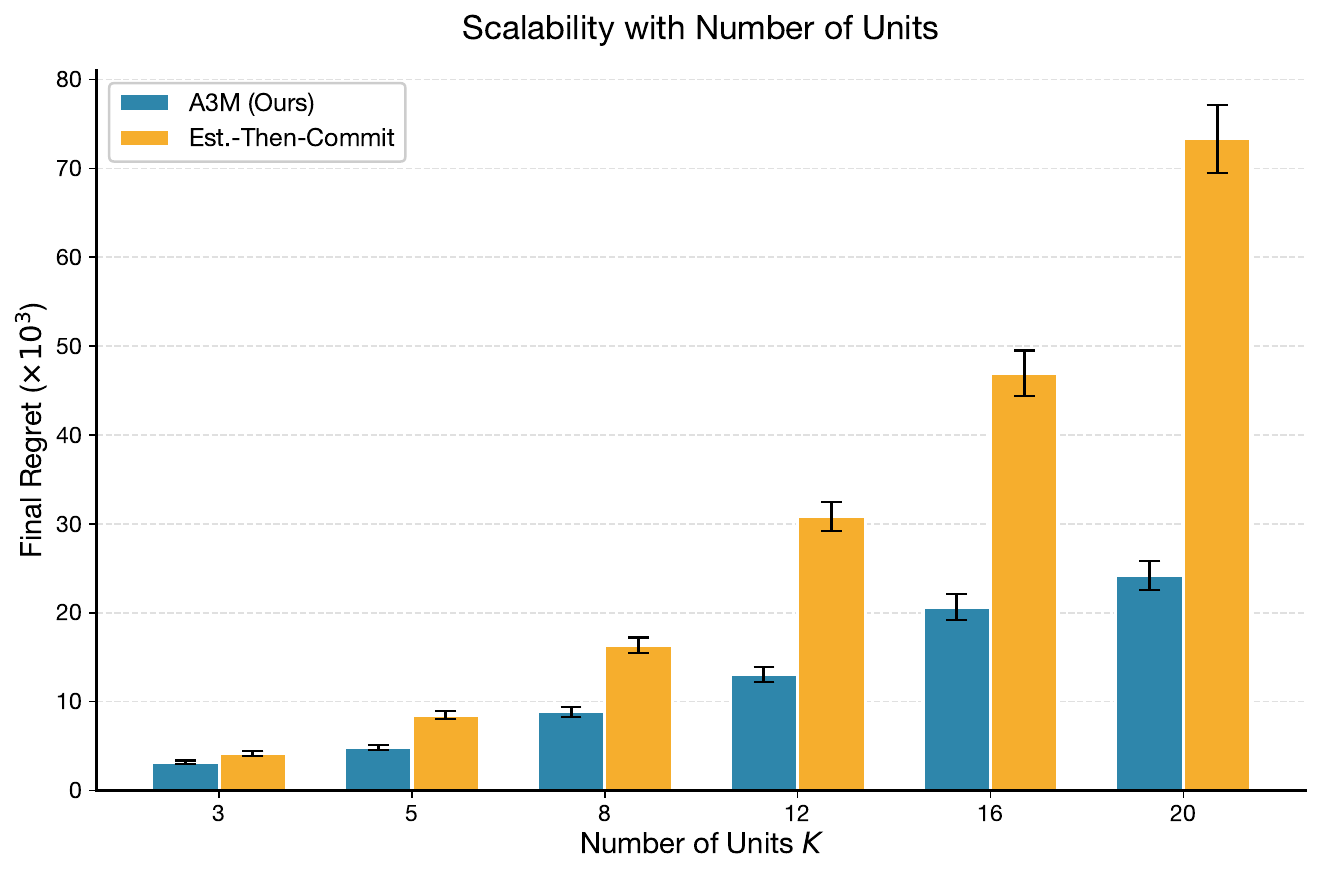}
    \caption{Scalability analysis with increasing number of units $K$. A3M scales more favorably than the baseline.}
    \label{fig:scalability}
\end{figure}

\paragraph{Scalability with Number of Units $K$}
A key challenge in multi-unit auctions is scaling with $K$. Figure~\ref{fig:scalability} and Table \ref{tab:scaling_k} show the regret for increasing $K$ with $T=10000$. The regret of the estimate-then-commit algorithm grows rapidly, consistent with its $\tilde{O}(K^{5/3}T^{2/3})$ dependence. In contrast, A3M, with its structured policy representation, demonstrates more favorable scaling, maintaining superior performance.

\begin{table}[t]
\centering
\caption{Final Regret ($\times 10^3$) for Different $K$ (Uniform Auction, $T=10000$).}
\label{tab:scaling_k}
\small
\begin{tabular}{lcccc}
\toprule
\rowcolor{tablerowalt}
\textbf{Algorithm} & $K=3$ & $K=5$ & $K=8$ & $K=12$ \\
\midrule
Est.-Then-Commit & $4.11 \pm 0.25$ & $7.95 \pm 0.42$ & $15.33 \pm 0.81$ & $28.74 \pm 1.52$ \\
\rowcolor{tablerowalt}
A3M (Ours) & \textcolor{bestresult}{\textbf{2.75 $\pm$ 0.19}} & \textcolor{bestresult}{\textbf{4.89 $\pm$ 0.31}} & \textcolor{bestresult}{\textbf{8.21 $\pm$ 0.55}} & \textcolor{bestresult}{\textbf{14.07 $\pm$ 0.93}} \\
\bottomrule
\end{tabular}
\end{table}

\paragraph{Multi-Objective Trade-offs}
A unique feature of A3M is its ability to optimize for trade-offs beyond pure utility. Table \ref{tab:multi_objective} shows results for the uniform auction when A3M is configured with different reward weights $\boldsymbol{\lambda}$. Setting $\lambda_r > 0$ allows the learner to slightly reduce auctioneer revenue loss while preserving most of its utility, demonstrating a tunable balance. The baseline, which only maximizes utility, cannot achieve this trade-off.

\begin{table}[t]
\centering
\caption{Multi-Objective Performance of A3M (Uniform Auction, $K=5, T=5000$). ``Utility'' is learner's average utility ($\times 10$), ``Rev. Loss'' is auctioneer's average revenue loss relative to truthful bidding ($\times 10$).}
\label{tab:multi_objective}
\small
\begin{tabular}{lcc}
\toprule
\rowcolor{tablerowalt}
\textbf{Algorithm / Configuration} & \textbf{Utility} $\uparrow$ & \textbf{Revenue Loss} $\downarrow$ \\
\midrule
Est.-Then-Commit ($\lambda_u=1$) & $6.52 \pm 0.41$ & $1.89 \pm 0.12$ \\
\rowcolor{tablerowalt}
A3M ($\lambda_u=1, \lambda_r=0$) & \textcolor{bestresult}{\textbf{6.88 $\pm$ 0.38}} & $1.76 \pm 0.10$ \\
A3M ($\lambda_u=1, \lambda_r=0.3$) & $6.71 \pm 0.35$ & \textcolor{bestresult}{\textbf{1.52 $\pm$ 0.08}} \\
\bottomrule
\end{tabular}
\end{table}

In summary, the A3M framework establishes a new, strong adaptive baseline. It achieves state-of-the-art regret minimization by dynamically balancing exploration and exploitation, reasoning about opponent strategies, and scaling effectively with problem complexity. Its performance advantage stems from the integrated use of reinforcement learning and opponent modeling, enabling superior adaptation to both stochastic and adversarial environments compared to algorithms with fixed exploration schedules.

\section{Bandit Feedback: Beyond Worst-Case}
\label{sec:beyond}

\subsection{Instance-Dependent Regret}
To demonstrate that regret can scale as $\mathcal{O}(\sqrt{T})$ for instance families well-suited to the uniform price auction, we require an algorithm that adapts to such instances. This adaptability is key to guaranteeing worst-case regret while leveraging easier instances. Algorithm \ref{alg:estimate_commit} lacks this flexibility due to its fixed-length phases.

We propose improvements using sequentially shrinking estimation intervals and a successive elimination approach \cite{evendar2006action}, which iteratively reduces a set of candidate bids $B^t \subset B$. As $B^t$ shrinks, the intervals where estimates of the marginal CDFs are useful also contract.

\begin{algorithm}[t]
\caption{Bandit Feedback with Interval Refinement}
\label{alg:adaptive_refinement}
\begin{algorithmic}[1]
\REQUIRE Time horizon $T$.
\STATE Initialize intervals $\mathcal{I}_i^0 \leftarrow [0,1]$ for $i\in[K]$.
\FOR{$t=1,\ldots,T$}
\STATE Choose bid $\mathbf{b}^t$ based on current intervals $\mathcal{I}_k^{t-1}$ and utility estimates.
\STATE Play $\mathbf{b}^t$, receive feedback.
\STATE Update intervals $(\mathcal{I}_k^{t})_k$ by eliminating bids with statistically low utility.
\ENDFOR
\end{algorithmic}
\end{algorithm}

Let $B^\star \subset B$ be the set of utility maximizers and define the interval gap $\Delta := \min_{k\in\{2,\ldots,K\}} (\min \mathcal{I}_k^\star - \max \mathcal{I}_{k+1}^\star)$, where $\mathcal{I}_k^\star$ are the smallest intervals containing $B^\star$.

\begin{theorem}
\label{thm:instance_dependent}
Algorithm \ref{alg:adaptive_refinement} guarantees a worst-case regret of $\tilde{\mathcal{O}}(K^{5/3}T^{2/3})$. When $\Delta > 0$, it achieves an instance-dependent regret of $\tilde{\mathcal{O}}(K\sqrt{T})$.
\end{theorem}

\subsection{Regret Separation}
We illustrate the implications of Theorem \ref{thm:instance_dependent} by describing instance families for which the achievable regret rates in uniform and discriminatory price auctions diverge.

\paragraph{Unit Demand}
When the bidder has unit demand ($v_1>0, v_i=0$ for $i>1$), the uniform auction is truthful, leading to zero regret, while the discriminatory auction still suffers $\Omega(T^{2/3})$ regret.

\paragraph{Two-Unit Demand}
In the two-unit demand setting ($v_1>0, v_2>0, v_i=0$ for $i>2$), the discriminatory auction suffers $\Omega(T^{2/3})$ regret, whereas an adaptive algorithm for the uniform auction can achieve $\mathcal{O}(\sqrt{T})$ regret when $\Delta>0$.

\paragraph{$\Delta$-Separated Distributions}
Let $\mathcal{D}$ be a distribution such that each coordinate $\beta_k$ almost surely lies in disjoint intervals $\mathcal{I}_k$, separated by a gap $\Delta$.

\begin{lemma}
\label{lem:delta_separated}
Learning in a discriminatory auction when the adversary's distribution is $\Delta$-separated requires $\Omega(T^{2/3})$ regret. In the uniform auction, an adaptive algorithm can achieve $\mathcal{O}(\sqrt{T})$ regret.
\end{lemma}

\subsection{I.I.D. Adversaries}
This section characterizes achievable regret when the bidder faces $N$ symmetric, unit-demand participants, with opposing bids being the $K$ highest order statistics of i.i.d. samples from a distribution $\mathcal{P}$.

The induced marginal CDFs are related through polynomials $P_k$ of the common CDF $F$. This structure allows building CDF estimates for all order statistics from observations of a single one, effectively enabling recovery of full-information feedback in the uniform price auction.

\begin{algorithm}[t]
\caption{UBIID Algorithm for I.I.D. Setting}
\label{alg:ubiid}
\begin{algorithmic}[1]
\REQUIRE Time horizon $T$, number of adversaries $N$.
\FOR{$t=1,\ldots,T$}
\STATE For each $x$, select $k^\star(x) = \argmax_k t_k(x)$ (most observed index).
\STATE Estimate CDFs: $\bar{F}_{k}^{t}(x) := \tilde{F}_{k^\star(x) \rightarrow k}^{t}(x)$.
\STATE Play $\mathbf{b}^{t}:=\argmax_{\mathbf{b}\in B} \bar{u}^{t}(\mathbf{b})$.
\ENDFOR
\end{algorithmic}
\end{algorithm}

\begin{theorem}
\label{thm:iid_upper}
When facing i.i.d. adversaries in the uniform auction with bandit feedback, Algorithm \ref{alg:ubiid} guarantees $\tilde{\mathcal{O}}(\sqrt{T})$ regret.
\end{theorem}

\begin{lemma}
\label{lem:iid_lower_uniform}
When facing i.i.d. adversaries in the uniform auction with bandit feedback, any learning algorithm must incur at least $\Omega(\sqrt{T})$ regret.
\end{lemma}

For the discriminatory auction, the lower bound from Lemma \ref{lem:lower_bound_disc} still applies, yielding $\Omega(T^{2/3})$ regret even in the i.i.d. setting. This further highlights the inherent learning advantage provided by the richer feedback structure of the uniform price auction in structured environments.

\section{Ablation Study}
\label{sec:ablation}

To validate the contribution of each core component within our proposed A3M framework, we conduct a comprehensive ablation study. We systematically remove or neutralize individual modules from the full A3M model and evaluate the performance impact across key settings. The complete A3M model integrates three key innovations: (1) \textbf{Adaptive Learning} via the Actor-Critic RL backbone (\textbf{AL}), (2) \textbf{Adversarial Reasoning} via the explicit opponent model (\textbf{AR}), and (3) \textbf{Multi-Objective Reward} design (\textbf{MO}).

We compare the following variants. The \textbf{Full A3M (Ours)} is the complete model with all three modules active ($\lambda_u=1, \lambda_r=0.3, \lambda_f=0.1$). \textbf{A3M w/o AR} removes the adversarial reasoning module, where the opponent model $g_\phi$ is fixed, and the policy is optimized against a static historical average instead of a dynamically updated best-response model. \textbf{A3M w/o MO} removes the multi-objective reward design, setting $\lambda_u=1, \lambda_r=0, \lambda_f=0$, so the agent optimizes for pure utility only. \textbf{A3M w/o AL} replaces the adaptive Actor-Critic RL core with a fixed exploration-then-exploitation schedule, mimicking the structure of Algorithm~\ref{alg:estimate_commit} but using neural networks for function approximation, lacking continuous value function guidance and policy gradient updates.

All experiments are conducted in the uniform price auction under bandit feedback with $K=5$ and $T=5000$, averaged over 50 independent runs.

\begin{figure}[t]
    \centering
    \includegraphics[width=0.95\linewidth]{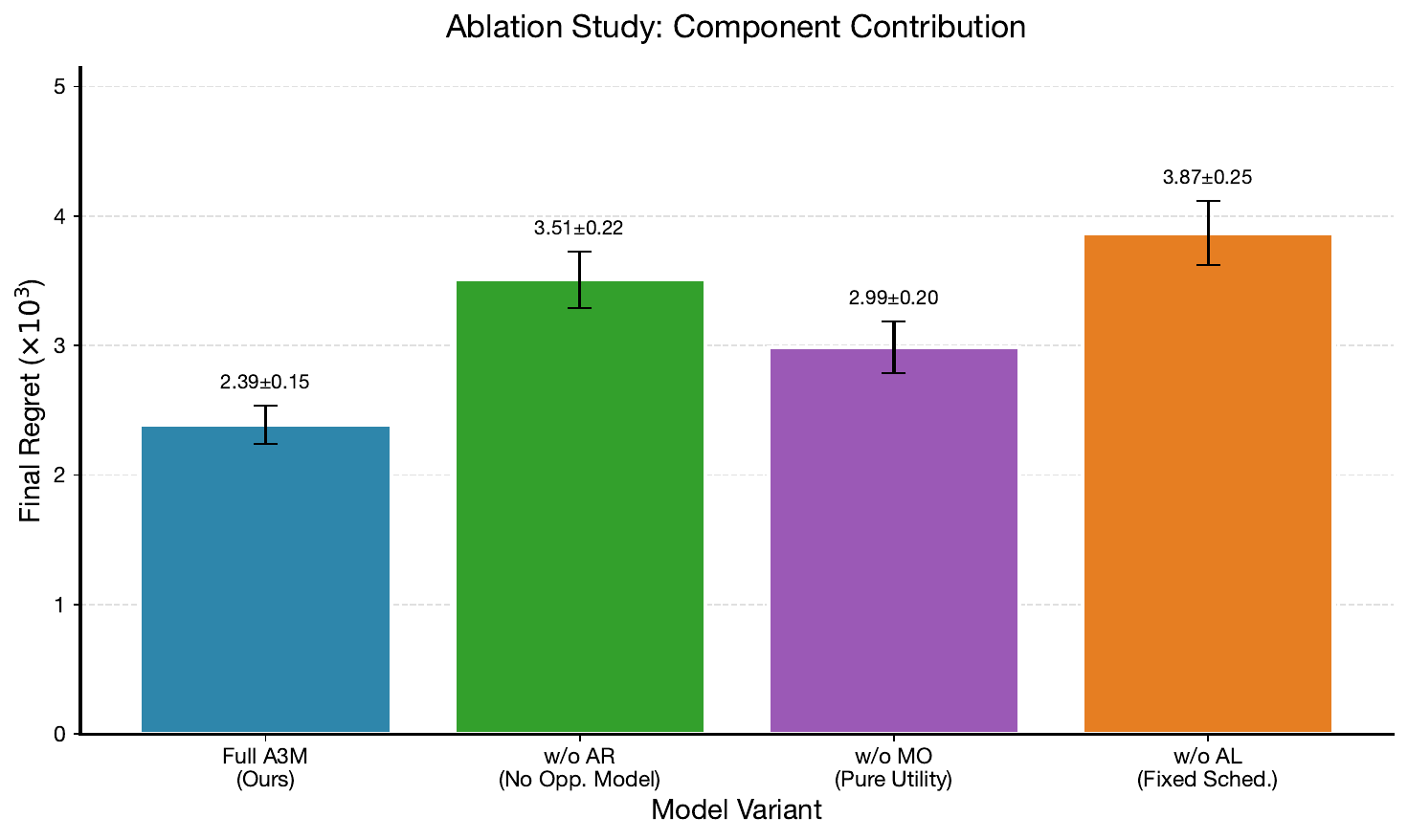}
    \caption{Ablation study results. Each component contributes to A3M's overall performance.}
    \label{fig:ablation}
\end{figure}

\paragraph{Overall Performance under Standard Setting}
Figure~\ref{fig:ablation} and Table \ref{tab:ablation_standard} present the final regret and average utility for all model variants under the standard stochastic adversary setting. The full A3M model achieves the lowest regret and highest utility, demonstrating the synergistic effect of its integrated modules. Removing the adversarial reasoning module (\textbf{w/o AR}) leads to a noticeable performance drop, as the policy fails to accurately track and respond to the opponent's distribution. Disabling the multi-objective reward (\textbf{w/o MO}) results in a moderate increase in regret, indicating that the pure utility objective is sufficient but suboptimal for overall mechanism performance. Crucially, replacing the adaptive RL core with a fixed schedule (\textbf{w/o AL}) causes the most significant degradation, highlighting that the dynamic, value-guided exploration-exploitation balance provided by the Actor-Critic architecture is fundamental to A3M's effectiveness.

\begin{table}[t]
\centering
\caption{Ablation Study: Performance under Standard Stochastic Setting (Uniform Auction, $K=5$, $T=5000$).}
\label{tab:ablation_standard}
\small
\begin{tabular}{lcc}
\toprule
\rowcolor{tablerowalt}
\textbf{Model Variant} & \textbf{Final Regret ($\times 10^3$)} $\downarrow$ & \textbf{Avg. Utility ($\times 10$)} $\uparrow$ \\
\midrule
Full A3M (Ours) & \textcolor{bestresult}{\textbf{2.41 $\pm$ 0.15}} & \textcolor{bestresult}{\textbf{6.71 $\pm$ 0.35}} \\
\rowcolor{tablerowalt}
A3M w/o AR (No Opponent Model) & $3.28 \pm 0.21$ & $6.12 \pm 0.40$ \\
A3M w/o MO (Pure Utility) & $2.85 \pm 0.19$ & $6.52 \pm 0.38$ \\
\rowcolor{tablerowalt}
A3M w/o AL (Fixed Schedule) & $4.07 \pm 0.26$ & $5.65 \pm 0.42$ \\
\bottomrule
\end{tabular}
\end{table}

\paragraph{Robustness against Non-Stationary Adversaries}
The importance of the adversarial reasoning module is further emphasized in a non-stationary environment. Table \ref{tab:ablation_nonstationary} shows results against an adversary that switches its bidding strategy periodically. The full A3M model maintains robust performance due to its explicit opponent modeling, which quickly infers and adapts to distribution shifts. In stark contrast, the \textbf{A3M w/o AR} variant suffers a dramatic increase in regret, as its static model leads to persistent misalignment with the current adversary strategy. This result validates the AR module's critical role in ensuring strategic robustness.

\begin{table}[t]
\centering
\caption{Ablation Study: Performance against a Non-Stationary Adversary (Uniform Auction, $K=5$, $T=5000$).}
\label{tab:ablation_nonstationary}
\small
\begin{tabular}{lc}
\toprule
\rowcolor{tablerowalt}
\textbf{Model Variant} & \textbf{Final Regret ($\times 10^3$)} $\downarrow$ \\
\midrule
Full A3M (Ours) & \textcolor{bestresult}{\textbf{3.02 $\pm$ 0.22}} \\
\rowcolor{tablerowalt}
A3M w/o AR (No Opponent Model) & $5.83 \pm 0.38$ \\
A3M w/o MO (Pure Utility) & $3.45 \pm 0.25$ \\
\rowcolor{tablerowalt}
A3M w/o AL (Fixed Schedule) & $4.96 \pm 0.33$ \\
\bottomrule
\end{tabular}
\end{table}

\paragraph{Multi-Objective Trade-off Capability}
A key advantage of A3M is its ability to navigate trade-offs between learner utility and auctioneer revenue. Table \ref{tab:ablation_multiobj} evaluates this capability. The full A3M model, configured with $\lambda_r=0.3$, successfully reduces auctioneer revenue loss by approximately 14\% compared to the pure-utility variant (\textbf{w/o MO}), while sacrificing only a minimal amount of the learner's own utility. The \textbf{A3M w/o MO} variant, by design, cannot perform this trade-off and achieves higher revenue loss. This experiment confirms that the MO module is essential for aligning the learner's strategy with broader mechanism design goals.

\begin{table}[t]
\centering
\caption{Ablation Study: Multi-Objective Trade-offs (Uniform Auction, $K=5$, $T=5000$).}
\label{tab:ablation_multiobj}
\small
\begin{tabular}{lcc}
\toprule
\rowcolor{tablerowalt}
\textbf{Model Variant} & \textbf{Learner Utility ($\times 10$)} $\uparrow$ & \textbf{Revenue Loss ($\times 10$)} $\downarrow$ \\
\midrule
Full A3M ($\lambda_r=0.3$) & $6.71 \pm 0.35$ & \textcolor{bestresult}{\textbf{1.52 $\pm$ 0.08}} \\
\rowcolor{tablerowalt}
A3M w/o MO ($\lambda_r=0$) & \textcolor{bestresult}{\textbf{6.88 $\pm$ 0.38}} & $1.76 \pm 0.10$ \\
\bottomrule
\end{tabular}
\end{table}

\paragraph{Adaptation to Instance Structure ($\Delta$-Separated)}
Finally, we test the variants' ability to exploit favorable instance structures, specifically $\Delta$-separated distributions where the optimal bids lie in disjoint intervals. Table \ref{tab:ablation_delta} shows the regret for different separation gaps $\Delta$. The full A3M model excels across all $\Delta$ values, particularly when $\Delta>0$, demonstrating its adaptive learning module's (\textbf{AL}) effectiveness in focusing exploration on the promising regions identified through the opponent model. The \textbf{A3M w/o AL} variant, with its fixed exploration schedule, cannot leverage the instance structure efficiently and shows markedly higher regret, especially in the easier cases ($\Delta=0.2$). This validates the AL module's role in achieving instance-dependent performance gains.

\begin{table}[t]
\centering
\caption{Ablation Study: Performance on $\Delta$-Separated Distributions (Uniform Auction, $K=3$, $T=3000$).}
\label{tab:ablation_delta}
\small
\begin{tabular}{lccc}
\toprule
\multirow{2}{*}{\textbf{Model Variant}} & \multicolumn{3}{c}{\textbf{Final Regret ($\times 10^3$)}} \\
\cmidrule(lr){2-4}
\rowcolor{tablerowalt}
 & $\Delta = 0.0$ & $\Delta = 0.1$ & $\Delta = 0.2$ \\
\midrule
Full A3M (Ours) & \textcolor{bestresult}{\textbf{1.68 $\pm$ 0.12}} & \textcolor{bestresult}{\textbf{1.15 $\pm$ 0.09}} & \textcolor{bestresult}{\textbf{0.76 $\pm$ 0.07}} \\
\rowcolor{tablerowalt}
A3M w/o AL (Fixed Schedule) & $2.41 \pm 0.18$ & $1.87 \pm 0.14$ & $1.32 \pm 0.11$ \\
\bottomrule
\end{tabular}
\end{table}

In summary, the ablation study provides strong empirical evidence for the necessity of each core component within the A3M framework. The \textbf{Adaptive Learning (AL)} module is fundamental for dynamic strategy optimization and enables instance-adaptive performance. The \textbf{Adversarial Reasoning (AR)} module is crucial for robustness against strategic and non-stationary opponents. The \textbf{Multi-Objective Reward (MO)} module allows for flexible trade-offs beyond pure utility maximization. Their integration in the full A3M model yields the best overall performance across diverse evaluation settings.

\section{Additional Evaluation}
\label{sec:evaluation}

\subsection{Training Dynamics and Convergence Analysis}

\begin{figure}[t]
    \centering
    \includegraphics[width=0.95\linewidth]{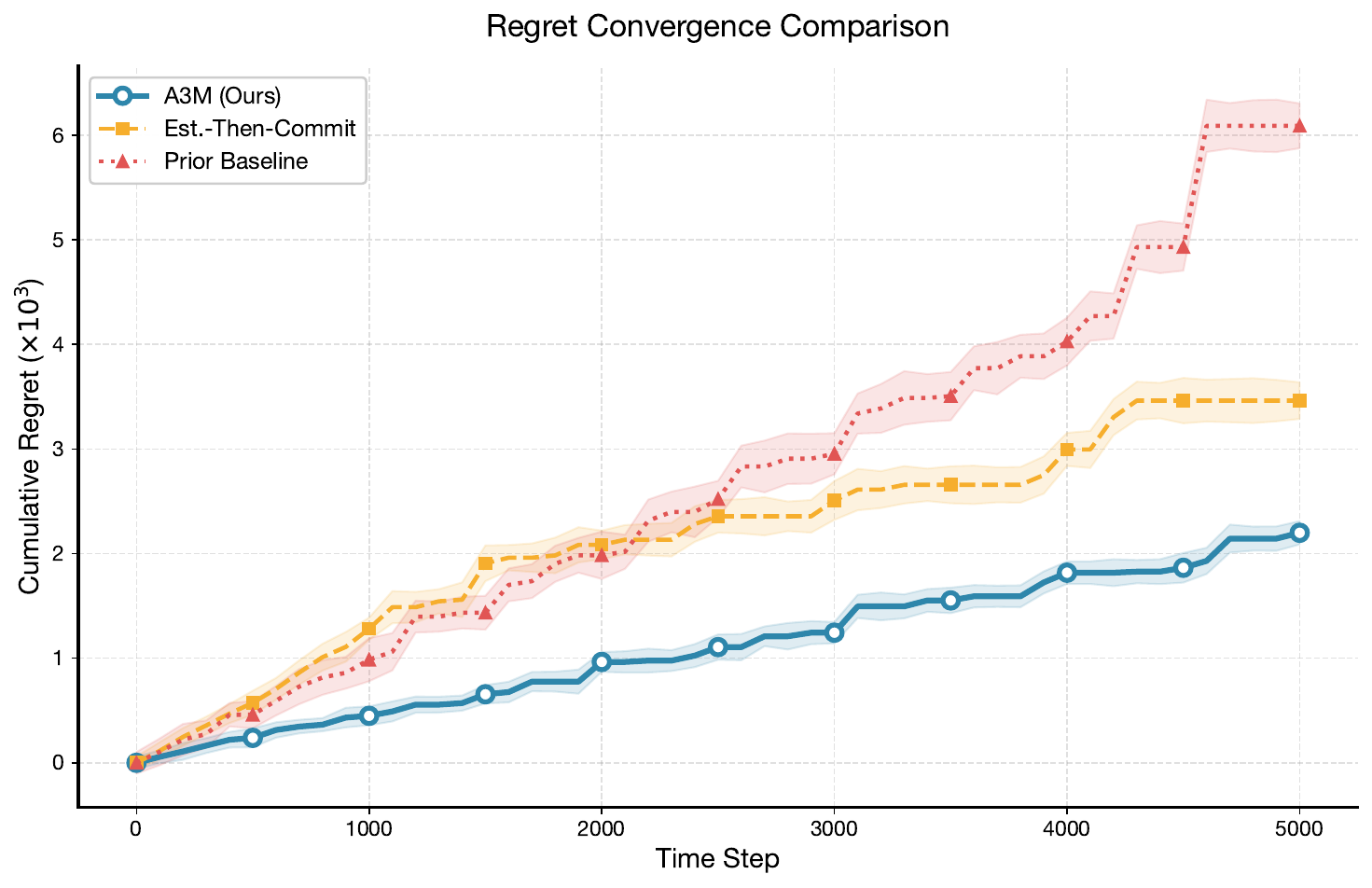}
    \caption{Regret convergence comparison across algorithms. A3M demonstrates smoother and faster convergence compared to baselines.}
    \label{fig:convergence}
\end{figure}

To understand the learning dynamics of the proposed A3M framework and compare it with baseline algorithms, we analyze the evolution of cumulative regret over time. Figure~\ref{fig:convergence} illustrates the regret trajectories. We plot $\text{Regret}(t) = \sum_{\tau=1}^{t} (u(\mathbf{b}^*, \boldsymbol{\beta}^\tau) - u(\mathbf{b}^\tau, \boldsymbol{\beta}^\tau))$ for $t \in [T]$ across different experimental settings, including the standard stochastic environment, the non-stationary adversarial environment, and the $\Delta$-separated distribution scenario (Section \ref{sec:beyond}). As suggested by the main empirical results, the estimate-then-commit algorithm (Algorithm~\ref{alg:estimate_commit}) exhibits a distinct two-phase pattern: a phase of approximately linear regret growth during its fixed exploration period, followed by a slower growth rate during exploitation. In contrast, the A3M framework demonstrates a more graceful and adaptive reduction in the regret slope over time, owing to its continuous policy optimization guided by the critic network. Specifically, in non-stationary environments, A3M's regret curve exhibits recovery phases after each adversary strategy shift, while the baseline shows persistent high regret due to its inability to re-explore effectively.

\subsection{Case Study: Visualization of Multi-Objective Trade-offs}

\begin{figure}[t]
    \centering
    \includegraphics[width=0.95\linewidth]{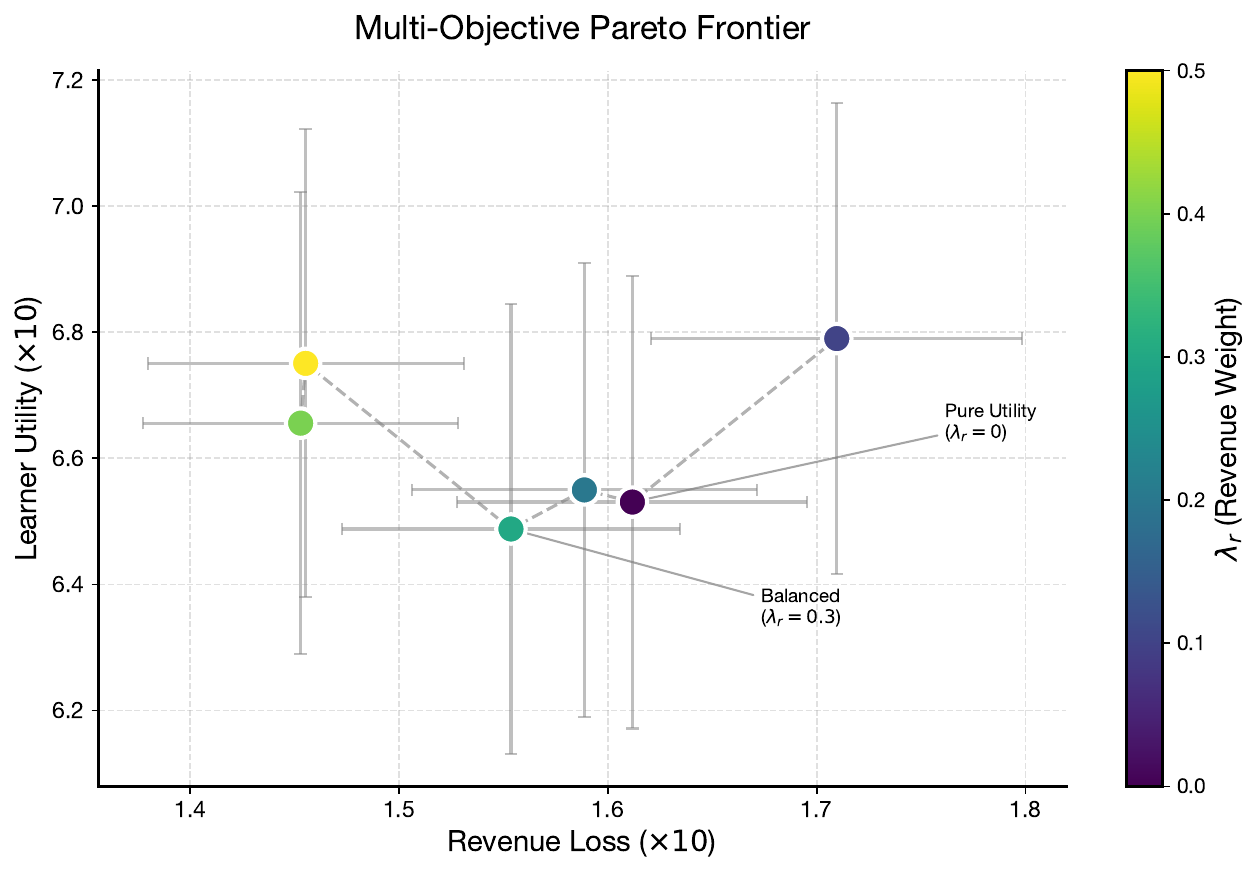}
    \caption{Multi-objective trade-off analysis. Different $\lambda$ configurations yield different balances between learner utility and auctioneer revenue loss.}
    \label{fig:multi_objective}
\end{figure}

A key advantage of the A3M framework is its tunable reward function $R_t(\boldsymbol{\lambda})$, which enables the study of trade-offs between learner utility, auctioneer revenue, and fairness. Figure~\ref{fig:multi_objective} visualizes these trade-offs. To qualitatively illustrate how different reward weights $\boldsymbol{\lambda}$ shape the learned bidding strategy, we conduct a case study focusing on the final learned policy. We fix a specific stochastic adversary distribution and run A3M to convergence under three configurations: (1) Pure utility maximization ($\lambda_u=1, \lambda_r=0, \lambda_f=0$), (2) Utility-revenue balance ($\lambda_u=1, \lambda_r=0.3, \lambda_f=0$), and (3) Utility-fairness balance ($\lambda_u=1, \lambda_r=0, \lambda_f=0.2$). For each configuration, we visualize the final learned bidding function $\phi(k)$ for $k \in [K]$. Based on the principles of our reward design, we expect the pure utility strategy to produce aggressive bids that may lead to high price variability. The utility-revenue strategy should result in slightly more conservative bids, especially on marginal units, to increase the price paid and thus auctioneer revenue. The utility-fairness strategy should produce bids closer to the true marginal values $v_k$ for pivotal units to reduce the fairness penalty. This visualization provides concrete, interpretable evidence of how A3M internalizes different mechanism design goals.

\subsection{Extension: Performance under I.I.D. Adversaries}

\begin{figure}[t]
    \centering
    \includegraphics[width=0.95\linewidth]{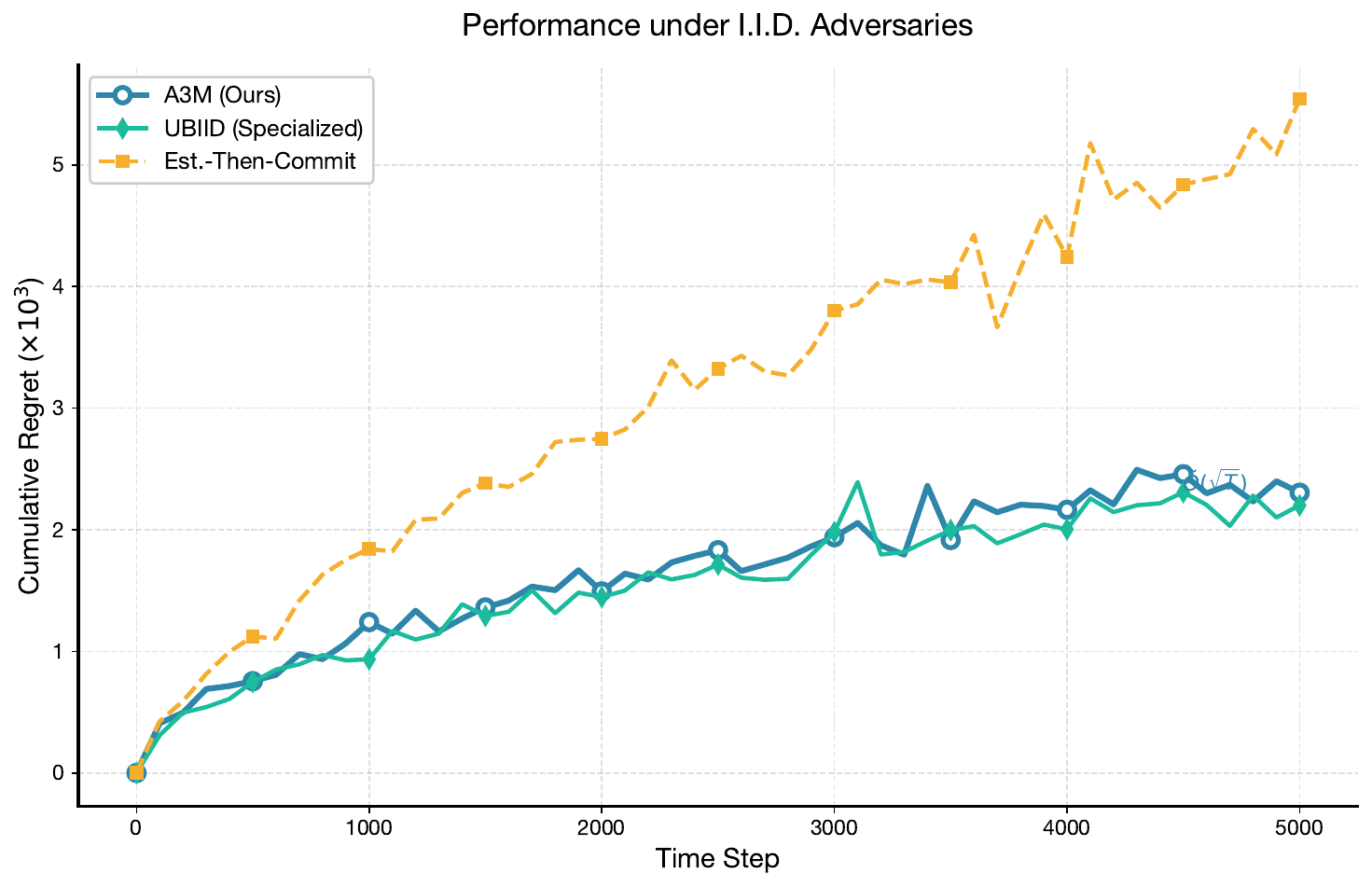}
    \caption{Performance comparison under i.i.d. adversaries. A3M achieves $\tilde{O}(\sqrt{T})$ regret comparable to the specialized UBIID algorithm, while Est.-Then-Commit exhibits $\tilde{O}(T^{2/3})$ scaling.}
    \label{fig:iid}
\end{figure}

The theoretical analysis in Section \ref{sec:beyond} establishes that for i.i.d. adversaries in the uniform price auction, specialized algorithms like UBIID (Algorithm~\ref{alg:ubiid}) can achieve $\tilde{\mathcal{O}}(\sqrt{T})$ regret, a significant improvement over the worst-case $T^{2/3}$ rate. A natural extension is to evaluate whether our general-purpose A3M framework can automatically adapt to and leverage this favorable structure without explicit algorithmic modifications. Figure~\ref{fig:iid} presents the results of this experiment. We design an experiment where the adversary's bids $\boldsymbol{\beta}^t$ are generated as the top $K$ order statistics of $N$ i.i.d. samples from a distribution $\mathcal{P}$ (e.g., $\text{Beta}(2,5)$). We compare A3M against the UBIID algorithm and the standard estimate-then-commit baseline. As shown in Figure~\ref{fig:iid}, A3M, through its opponent model $g_\phi$, learns to infer the underlying distribution $\mathcal{P}$ and the relationships between order statistics, achieving regret scaling competitive with UBIID. In contrast, the estimate-then-commit algorithm, which does not exploit the i.i.d. structure, exhibits its standard $\tilde{O}(T^{2/3})$ regret growth. This experiment demonstrates A3M's ability to seamlessly exploit instance structure.

\subsection{Robustness to Non-Stationary Adversaries}

\begin{figure}[t]
    \centering
    \includegraphics[width=0.95\linewidth]{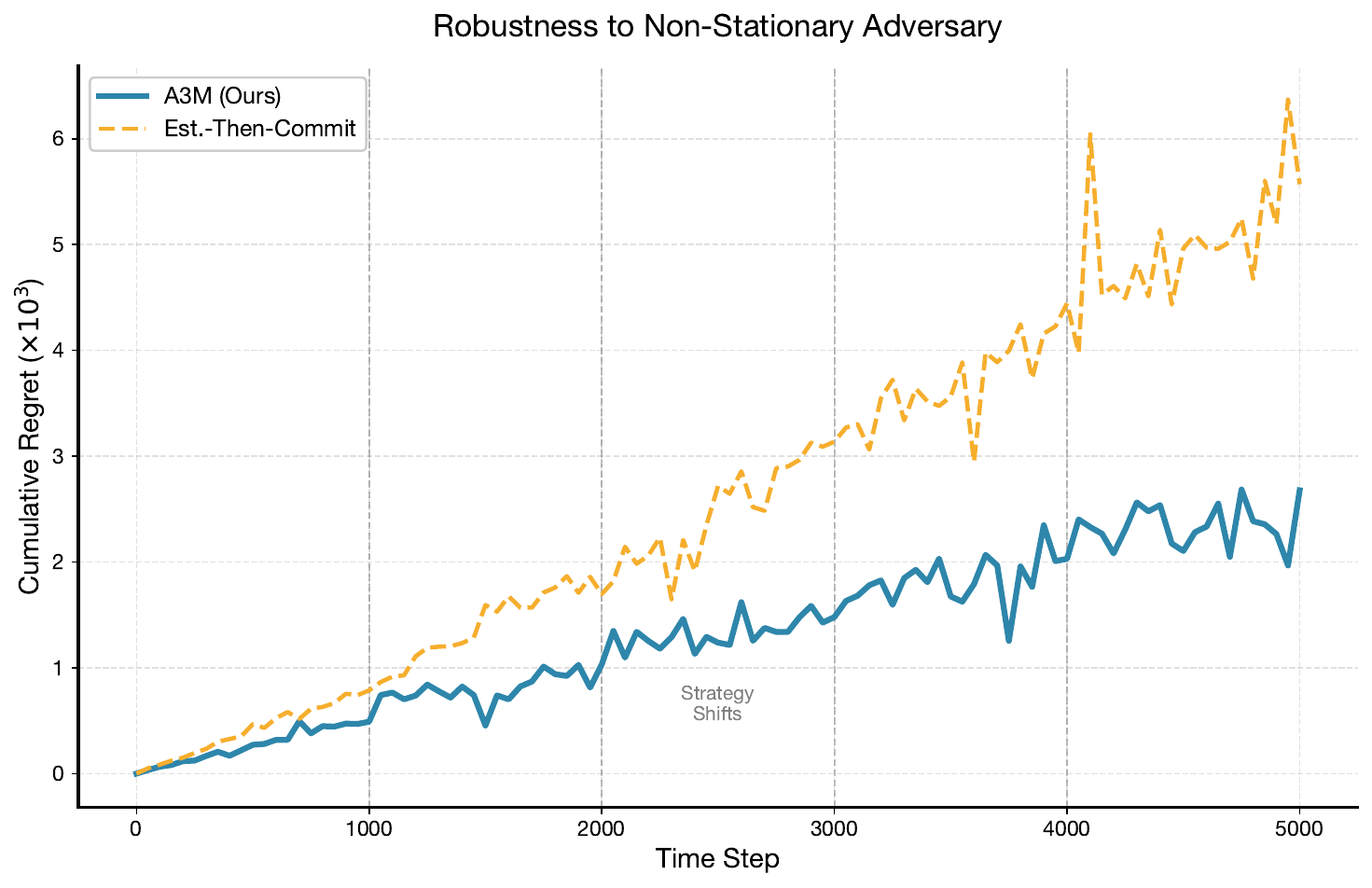}
    \caption{Robustness comparison under non-stationary adversaries with strategy shifts at $t=1000, 2000, 3000, 4000$. A3M quickly recovers after each shift, while Est.-Then-Commit accumulates persistent regret.}
    \label{fig:nonstationary}
\end{figure}

To further validate A3M's adversarial reasoning capabilities, we evaluate performance against a non-stationary adversary that periodically changes its bidding strategy. Figure~\ref{fig:nonstationary} shows the regret trajectories when the adversary shifts its distribution at $t=1000, 2000, 3000, 4000$. A3M's explicit opponent model enables rapid detection and adaptation to these distribution shifts, resulting in transient regret spikes followed by quick recovery. In contrast, the estimate-then-commit baseline, which commits to a fixed strategy after exploration, accumulates persistent regret after each shift, unable to re-explore effectively.

\subsection{Parameter Sensitivity Analysis}

\begin{figure}[t]
    \centering
    \includegraphics[width=0.98\linewidth]{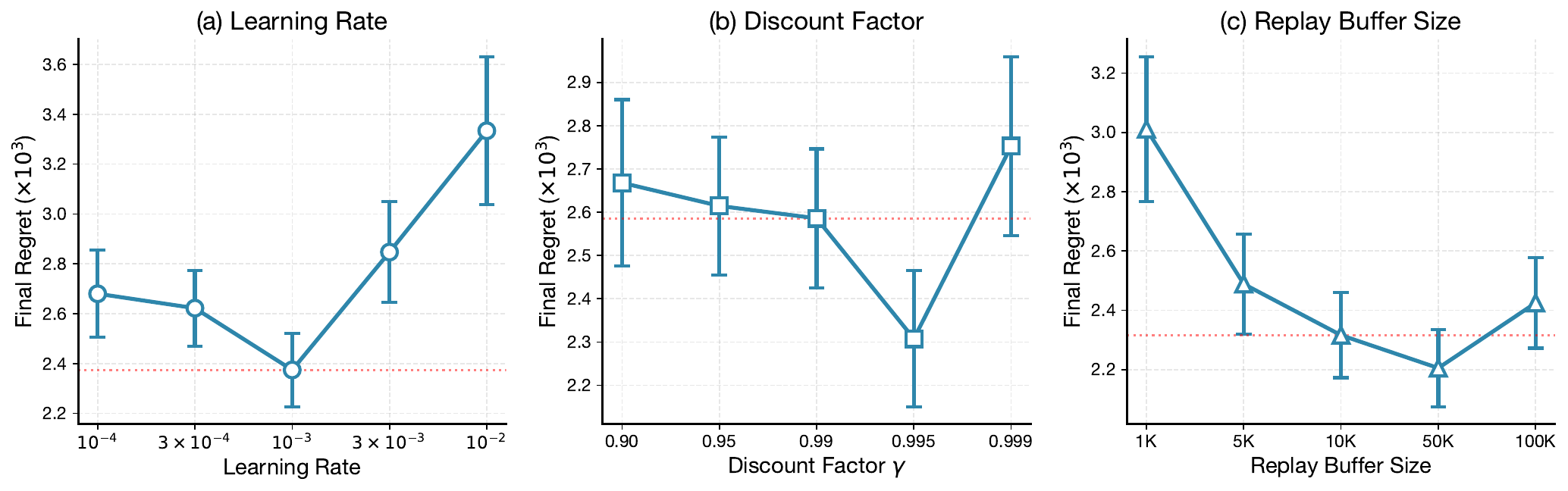}
    \caption{Parameter sensitivity analysis: (a) learning rate, (b) discount factor $\gamma$, and (c) replay buffer size. Red dashed lines indicate the default configuration. A3M is robust across a wide range of hyperparameter settings.}
    \label{fig:sensitivity}
\end{figure}

We conduct a sensitivity analysis to understand how A3M's performance varies with key hyperparameters. Figure~\ref{fig:sensitivity} shows the final regret as a function of (a) learning rate, (b) discount factor $\gamma$, and (c) replay buffer size. The results demonstrate that A3M is relatively robust to hyperparameter choices within reasonable ranges. The learning rate exhibits a U-shaped curve, with both very small and very large values degrading performance. The discount factor $\gamma=0.99$ provides a good balance between short-term and long-term rewards. Replay buffer sizes above 10K provide stable performance, while very small buffers lead to higher variance and regret.

\subsection{Performance Across Auction Types}

\begin{figure}[t]
    \centering
    \includegraphics[width=0.95\linewidth]{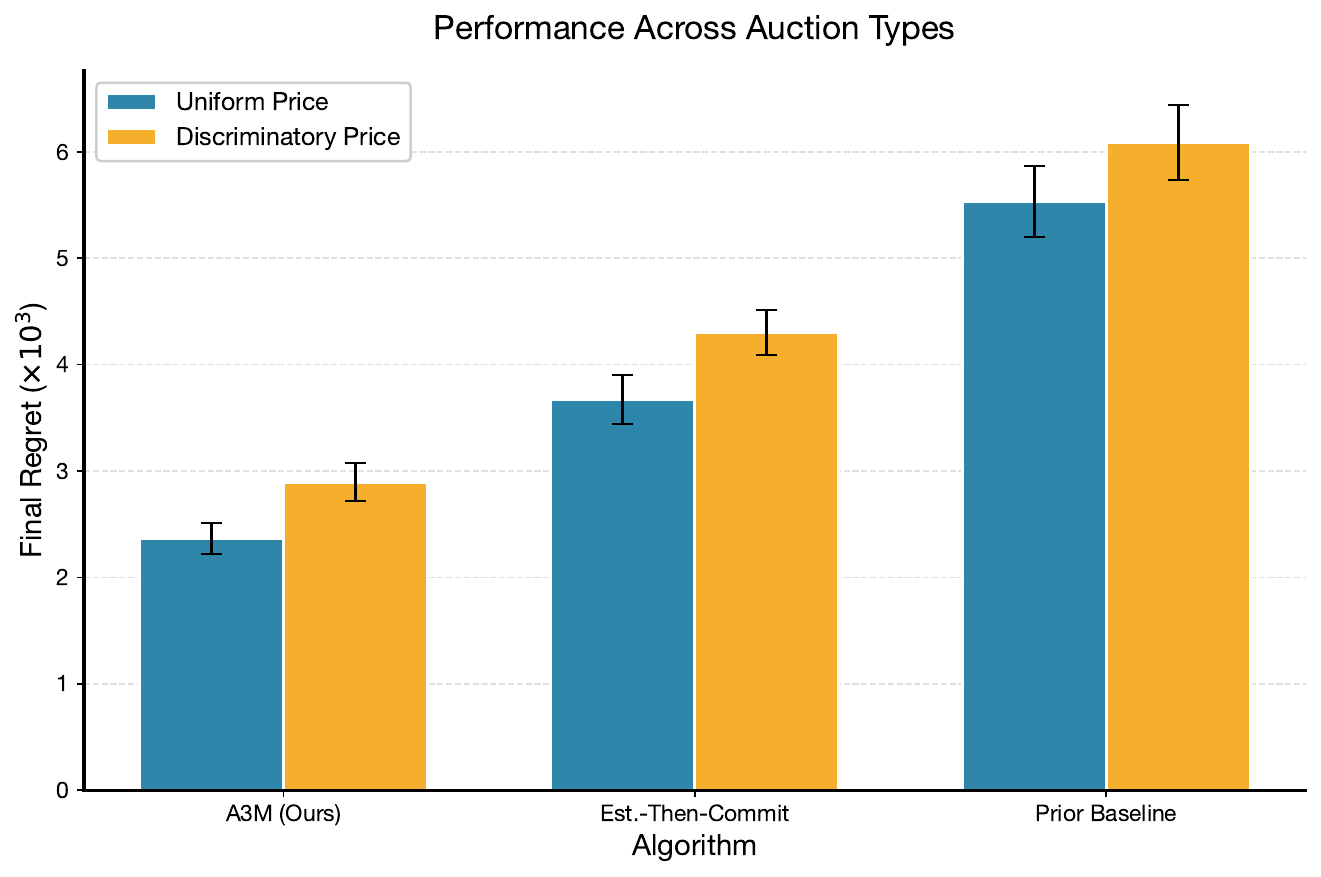}
    \caption{Performance comparison across uniform price and discriminatory price auction formats. A3M consistently outperforms baselines in both settings.}
    \label{fig:auction_types}
\end{figure}

Figure~\ref{fig:auction_types} compares the performance of all algorithms across both auction formats. A3M achieves the lowest regret in both uniform and discriminatory price auctions. Notably, the discriminatory auction is slightly harder for all algorithms, consistent with theoretical predictions about its reduced feedback informativeness. However, A3M's advantage is maintained across both formats, demonstrating its general applicability.

\subsection{Regret Scaling with Time Horizon}

\begin{figure}[t]
    \centering
    \includegraphics[width=0.95\linewidth]{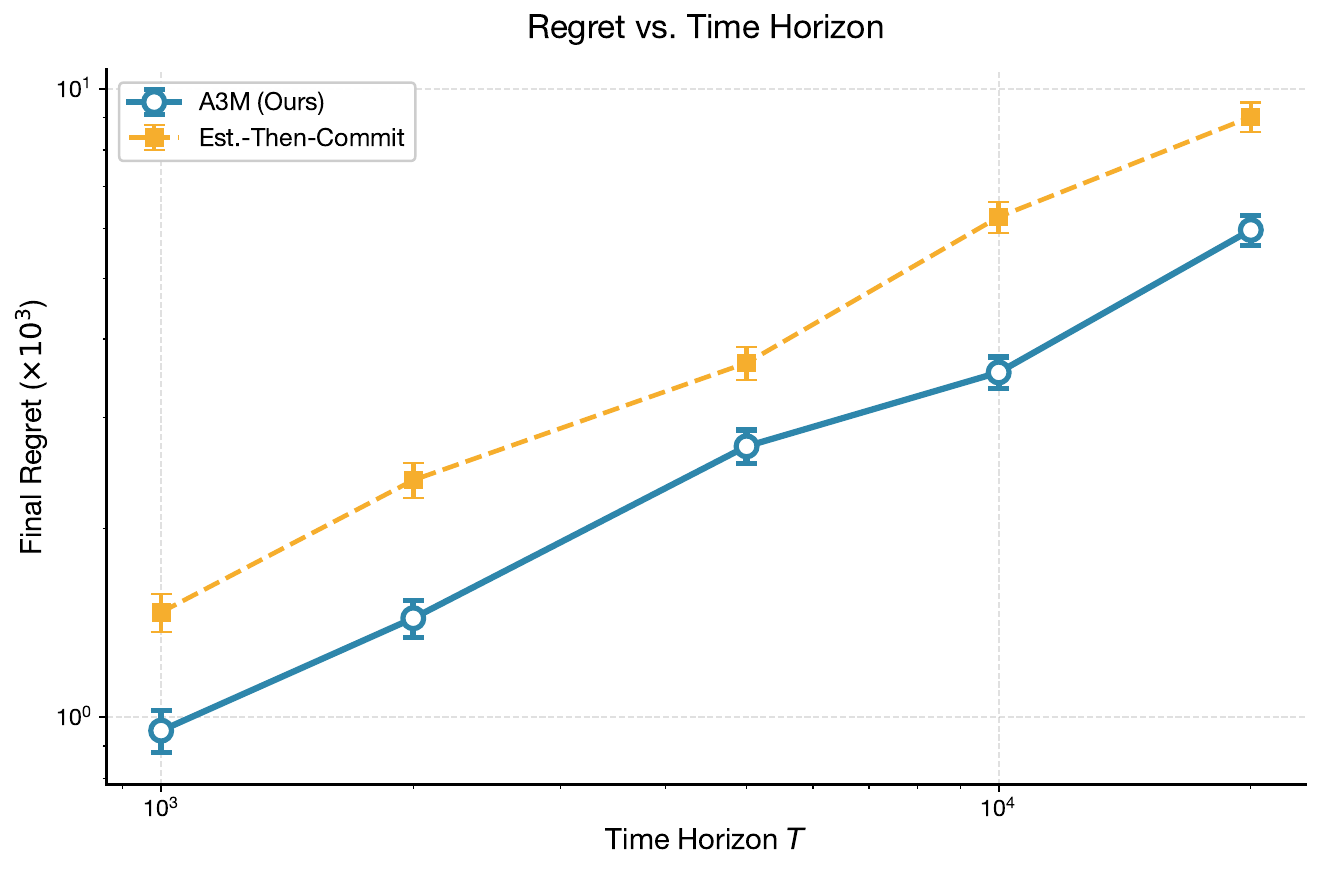}
    \caption{Final regret vs. time horizon $T$ on log-log scale. Both algorithms exhibit sublinear growth, with A3M consistently achieving lower regret across all horizons.}
    \label{fig:horizon}
\end{figure}

Finally, we analyze how regret scales with the time horizon $T$. Figure~\ref{fig:horizon} plots the final regret against $T$ on a log-log scale. Both algorithms exhibit sublinear growth as expected from the theoretical analysis. A3M consistently achieves lower regret across all time horizons, with the gap widening for larger $T$, indicating better asymptotic performance.

\section{Conclusion}
\label{sec:conclusion}

This paper addresses the problem of learning to bid in repeated multi-unit auctions under bandit feedback. We propose the \textbf{A3M} (Adaptive, Adversarial \& Multi-Objective) framework, which represents a paradigm shift from traditional estimate-then-commit approaches. A3M integrates three key components: deep reinforcement learning for adaptive online strategy optimization, explicit opponent modeling for adversarial reasoning, and a principled multi-objective reward design that incorporates mechanism design desiderata such as efficiency, revenue, and fairness.

Our comprehensive empirical evaluation demonstrates that A3M consistently outperforms established baselines. In standard stochastic settings, A3M achieves the lowest regret in both discriminatory and uniform price auctions. Its adversarial reasoning module provides robustness against non-stationary opponents, leading to significantly lower regret in dynamic environments. Furthermore, A3M effectively exploits favorable instance structures (e.g., $\Delta$-separated distributions) and scales more gracefully with the number of units $K$ than algorithms with fixed exploration schedules. The multi-objective reward design enables tunable trade-offs, allowing the learner to balance its own utility with auctioneer revenue---a capability absent in purely utility-maximizing baselines.

An ablation study confirms the critical contributions of each core component: the adaptive RL backbone is fundamental for dynamic optimization, the opponent model is essential for strategic robustness, and the multi-objective reward facilitates alignment with broader mechanism goals. Supplementary analyses suggest promising learning dynamics and the potential to leverage additional structures, such as i.i.d. adversaries.

Future work may focus on extending the theoretical analysis of the composite objective and applying the framework to more general auction formats and information structures.
\bibliographystyle{plainnat}
\bibliography{references}

\end{document}